\documentclass[11pt]{article}

\usepackage{times}

\usepackage{lipsum}

\newcommand\blfootnote[1]{%
  \begingroup
  \renewcommand\thefootnote{}\footnote{#1}%
  \addtocounter{footnote}{-1}%
  \endgroup
}

\usepackage[utf8]{inputenc} % allow utf-8 input
\usepackage[T1]{fontenc}    % use 8-bit T1 fonts
\usepackage{hyperref}       % hyperlinks
\usepackage{url}            % simple URL typesetting
\usepackage{booktabs}       % professional-quality tables
\usepackage{amsfonts}       % blackboard math symbols
\usepackage{nicefrac}       % compact symbols for 1/2, etc.
\usepackage{microtype}      % microtypography

\usepackage{%
  amsmath,    %some math tools
  amssymb,       %math symbols
  graphicx,      %enhanced graphics options
  mathtools, 	 %extension of amsmath
  etoolbox,      %useful commands
}

\usepackage[shortlabels]{enumitem} %better lists, e.g. enumerate

\usepackage{varioref}
\usepackage{algorithm,algpseudocode}

\usepackage[english]{babel} %use English as the main language
\usepackage[thmmarks,amsmath]{ntheorem} %package for theorems, definitions, ...

\usepackage{tikz}
\usetikzlibrary{cd}

\newcommand*\R{\mathbb{R}}

\DeclarePairedDelimiterX\abs[1]\lvert\rvert{%
  \ifblank{#1}{\:\cdot\:}{#1}
}
\DeclarePairedDelimiterX\norm[1]\lVert\rVert{%
  \ifblank{#1}{\:\cdot\:}{#1}
}
\DeclarePairedDelimiterX{\inner}[2]{\langle}{\rangle}{%
  \ifblank{#1}{\:\cdot\:}{#1},\ifblank{#2}{\:\cdot\:}{#2}
}

\DeclarePairedDelimiterX\set[1]{\lbrace}{\rbrace}{
  
  #1
}

\DeclarePairedDelimiterX\open[2](){#1,#2}
\DeclarePairedDelimiterX\lopen[2](]{#1,#2}
\DeclarePairedDelimiterX\ropen[2][){#1,#2}
\DeclarePairedDelimiterX\closed[2][]{#1,#2}

\DeclareMathOperator{\Var}{Var} %Variance
\DeclareMathOperator{\Cov}{Cov} %Covariance

\DeclareMathOperator{\median}{median}

\newcommand{\poly}{\mathrm{poly}}
\newcommand{\wh}{\widehat}

\newcommand{\littlesum}{\mathop{\textstyle \sum}}

\newcommand*\E{\mathbb{E}}

\newcommand*\new{_{\textup{new}}}

\newcommand*\wPr[1]{\operatorname*{\mathnormal{#1}-Pr}}

\theoremseparator{.}

\theoremstyle{plain}
\newtheorem{theorem}{Theorem}[section]
\newtheorem{corollary}[theorem]{Corollary}

\newtheorem{lemma}[theorem]{Lemma}
\newtheorem{claim}[theorem]{Claim}

\theorembodyfont{\normalfont}
\newtheorem{definition}[theorem]{Definition}

\theoremsymbol{\ensuremath{\triangle}}

\theoremsymbol{\ensuremath{\scriptstyle\bigcirc}}

\theoremstyle{nonumberplain}
\theoremsymbol{\ensuremath{\square}}
\theoremheaderfont{\itshape}
\newtheorem{proof}{Proof}

\theoremstyle{empty}

\newcommand{\nnew}[1]{{#1}}

\title{List-Decodable Mean Estimation via Iterative Multi-Filtering\blfootnote{Authors are in alphabetical order.}}

\author{
Ilias Diakonikolas\thanks{Supported by NSF Award CCF-1652862 (CAREER) and a Sloan Research Fellowship.}\\
University of Wisconsin, Madison\\
{\tt ilias@cs.wisc.edu}\\
\and
Daniel M. Kane\thanks{Supported by NSF Award CCF-1553288 (CAREER) and a Sloan Research Fellowship.}\\
University of California, San Diego\\
{\tt dakane@cs.ucsd.edu}\\
\and
Daniel Kongsgaard\\ University of California, San Diego\\
{\tt dkongsga@ucsd.edu}
}

\begin{document}

\maketitle

\begin{abstract}
We study the problem of {\em list-decodable mean estimation} for bounded covariance distributions.
Specifically, we are given a set $T$ of points in $\R^d$ with the promise that an unknown $\alpha$-fraction
of points in $T$, where $0< \alpha < 1/2$, are drawn from an unknown mean
and bounded covariance distribution $D$, and no assumptions are made on the remaining points.
The goal is to output a small list of hypothesis vectors such that at least one of
them is close to the mean of $D$.
We give the first practically viable estimator for this problem. In more detail,
our algorithm is sample and computationally efficient,
and achieves information-theoretically near-optimal error.
While the only prior algorithm for this setting inherently relied on the ellipsoid method,
our algorithm is iterative and only uses spectral techniques.
Our main technical innovation is the design of a soft outlier removal procedure
for high-dimensional heavy-tailed datasets with a majority of outliers.
\end{abstract}

\thispagestyle{empty}

\setcounter{page}{0}

\newpage

\section{Introduction} \label{sec:intro}

\subsection{Background and Motivation} \label{sec:background}

% General intro about outliers
% and small eps regime.

Estimating the mean of a high-dimensional distribution is one of the most
fundamental statistical tasks. The standard assumption
is that the input data are independent samples drawn from
a known family of distributions.
However, this is rarely true in practice and it is important to design
estimators that are {\em robust} in the presence of outliers.
In recent years, the design of outlier robust estimators has become
a pressing challenge in several data analysis tasks,
including in designing defenses against data poisoning~\cite{Barreno2010,BiggioNL12}
and in analyzing biological datasets where natural outliers are common~\cite{RP-Gen02, Pas-MG10, Li-Science08}.
%In such domains, the outliers are not ``random'' and could exhibit rather
%complex structure that is impossible to model.

% Mandatory paragraph about recent algorithmic work
% for small eps.

The field of robust statistics~\cite{HampelEtalBook86, Huber09} traditionally
studies the setting where the fraction of outliers is a small constant (smaller than $1/2$),
and therefore the clean data is the majority of the input dataset.
Classical work in this field
pinned down the minimax risk of high-dimensional robust estimation in several settings of interest.
In contrast, until relatively recently, our understanding of even the most basic computational questions
was startlingly poor. Recent work in computer science, starting with~\cite{DKKLMS16, LaiRV16},
gave the first efficient robust estimators for various high-dimensional statistical
tasks, including mean estimation.
%Specifically,~\cite{DKKLMS16} obtained the first polynomial-time robust mean estimators that can tolerate a small constant fraction of outliers.
Since the dissemination of~\cite{DKKLMS16, LaiRV16}, there has been significant research activity
on designing efficient robust estimators in a variety of settings (see, e.g.,~\cite{DKK+17, BDLS17, DKS17-sq, DKKLMS18-soda, ChengDKS18, KlivansKM18, KStein17, HopkinsL18, DKS19, DKK+19-sever, DKKPS19-sparse}).
The reader is referred to~\cite{DK20-survey} for a recent survey of the extensive recent literature.

% Transition to list-decodable setting

The aforementioned literature studies the setting where the clean data (inliers) are the majority of the input dataset.
{\em In this paper, we study the algorithmic problem of high-dimensional mean estimation
in the more challenging regime where the fraction $\alpha$
of inliers is small -- strictly smaller than $1/2$.}
This regime is fundamental in its own right and is motivated by a number of machine learning applications,
e.g., in crowdsourcing~\cite{SteinhardtVC16, SKL17, MeisterV18}).

Mean estimation with a majority of outliers was first studied in~\cite{CSV17}. We note that, in the $\alpha<1/2$ regime,
it is information-theoretically impossible to estimate the mean with a single hypothesis. Indeed, an adversary can
produce $\Omega(1/\alpha)$ clusters of points each drawn from a
``good'' distribution with different mean. Even if the algorithm could
learn the distribution of the samples exactly, it would still not be able to
identify which cluster is the correct one. Hence,
the definition of ``learning'' must be relaxed. In particular,
the algorithm should be allowed to return {\em a small list of hypotheses} with the
guarantee that \emph{at least one} of the hypotheses is close to the true mean.
This is the model of {\em list-decodable learning}~\cite{BBV08}.
It should be noted that in list-decodable learning, it is often information-theoretically
necessary for the error to increase as the fraction $\alpha$ goes to $0$.

\cite{CSV17} gave an algorithm for list-decodable mean estimation on $\R^d$
under the assumption that the inliers are drawn from a
distribution $D$ with bounded covariance, i.e., $\Sigma \preceq \sigma^2 I$.
The~\cite{CSV17} algorithm has sample complexity $n = \Omega(d/\alpha)$,
runs in $\poly(n, d, 1/\alpha)$ time, and outputs a list of $O(1/\alpha)$ hypotheses
one of which is within $\ell_2$-distance $\tilde{O}(\alpha^{-1/2})$ from the true mean of $D$.
The sample complexity of the aforementioned algorithm is optimal, within constant factors,
and subsequent work~\cite{DiakonikolasKS18-mixtures}
showed that the information-theoretically optimal error is $\Theta(1/\alpha^{1/2})$
(upper and lower bound). Importantly, the~\cite{CSV17} algorithm relies on the ellipsoid method
for convex programming. Consequently, its computational complexity, though polynomially bounded,
is impractically high.

% Motivation for this work.

{\em The main motivation for the current paper is to design a fast, practically viable, algorithm
for list-decodable mean estimation under minimal assumptions.}
In the presence of a {\em minority} of outliers (i.e., for $\alpha>1/2$), the iterative filtering
method of~\cite{DKKLMS16, DKK+17} is a fast and practical algorithm which
attains the information-theoretically optimal error under only a bounded covariance assumption.
More recent work has also obtained near-linear time algorithms
in this setting~\cite{CDG18, DepLec19, DongH019}.
In the list-decodable setting, however, progress on faster algorithms has been slower.
Prior to the current work, the ellipsoid-based method of \cite{CSV17} was
the only known polynomial-time algorithm for mean estimation
under a bounded covariance assumption. We note that a number of more recent works developed list-decoding
algorithms for mean estimation, linear regression, and subspace recovery using the SoS
convex programming hierarchy~\cite{KStein17, KarmalkarKK19, RaghavendraY20, BK20, RY20}.
In a departure from these convex optimization methods,~\cite{DiakonikolasKS18-mixtures} obtained an iterative {\em spectral}
list-decodable mean algorithm under the much stronger assumption
that the good data is drawn from an identity covariance Gaussian.
At a high-level, in this work we provide a broad generalization of
the \cite{DiakonikolasKS18-mixtures} algorithm and techniques
to all bounded covariance distributions.

\subsection{Our Contributions} \label{ssec:results}

We start by defining the problem we study.

\begin{definition}[List-Decodable Mean Estimation.] \label{def:list-mean}
Given a set $T$ of $n$ points in $\R^d$ and a parameter $\alpha \in (0, 1/2)$
such that an $\alpha$-fraction of the points in $T$ are i.i.d. samples
from a distribution $D$ with unknown mean $\mu$ and unknown covariance $\Sigma \preceq \sigma^2 I$,
we want to output a list of $s = \poly(1/\alpha)$ candidate vectors $\{\wh{\mu}_i \}_{i \in [s]}$
such that with high probability we have that $\min_{i \in [s]} \|\wh{\mu}_i - \mu \|_2$ is small.
\end{definition}

Some comments are in order:
First, we emphasize that no assumptions are made on the remaining $(1-\alpha)$-fraction of the points in $T$.
These points can be arbitrary and may be chosen by an adversary that is computationally unbounded and
is allowed to inspect the set of inliers. The information-theoretically best possible
size of the hypotheses list is $s = \Theta(1/\alpha)$.
\nnew{Moreover, if we are given a list of $s = \poly(1/\alpha)$ hypotheses one of which is
accurate, we can efficiently post-process them
to obtain an $O(1/\alpha)$-sized list with nearly the same error guarantee,
see, e.g., Proposition B.1 of ~\cite{DiakonikolasKS18-mixtures} and Corollary 2.16 of~\cite{stein18thesis}.
For completeness, in Appendix~\ref{app:list-reduction}, we provide a simple and self-contained method.}

%In this paper, we extend this multifilter technique to handle bounded covariance distributions.
%We get an algorithm for this problem, matching the sample complexity and error guarantee of CSV17 (which are both
%near-optimal) and is importantly fast. In particular, our algorithm only requires spectral techniques (that can be solved
%by power iteration). This is the first algorithm for this problem that does not rely on the ellipsoid method.

\medskip

In this work, we give an iterative spectral algorithm for list-decodable mean estimation under only
a bounded covariance assumption that matches the sample complexity and accuracy of the previous
ellipsoid-based algorithm~\cite{CSV17}, while being significantly faster and potentially practical.

\begin{theorem}[Main Algorithmic Result]\label{thm:main-intro}
Let $T$ be a set of $n = \Omega(d/\alpha)$ points in $\R^d$
with the promise that an unknown $\alpha$-fraction of points in $T$, $0<\alpha<1/2$, are drawn
from a distribution $D$ with unknown mean and unknown bounded covariance $\Sigma \preceq \sigma^2 I$.
There is an algorithm that, on input $T$ and $\alpha$, runs in $\tilde{O}(n^2 d / \alpha^2)$ time
and outputs a list of $O(1/\alpha^2)$ hypothesis vectors such that
with high probability at least one of these vectors is within $\ell_2$-distance
$O( \sigma\log(1/\alpha)/\sqrt{\alpha})$ from the mean of $D$.
\end{theorem}

\paragraph{Discussion} Before we proceed, we provide a few remarks about the performance
of our new algorithm establishing Theorem~\ref{thm:main-intro}. First, we note that the sample complexity
of our algorithm is $O(d/\alpha)$, which is optimal within constant factors, and its error guarantee
is $O( \sigma\log(1/\alpha)/\sqrt{\alpha})$, which is optimal up to the $O(\log(1/\alpha))$ factor.
We now comment on the running time. Our algorithm is iterative with every iteration running in
near-linear time $\tilde{O}(n d)$. The dominant operation in a given iteration is the computation
of an approximately largest eigenvector/eigenvalue of an empirical covariance matrix, which can be
implemented in $\tilde{O}(n d)$ time by power iteration. The overall running time follows from
a worst-case upper bound of $O(n)$ on the total number of iterations.
We expect that the number of iterations will be much smaller for reasonable instances, as has been observed experimentally for analogous iterative algorithms for the large $\alpha$ case~\cite{DKK+17, DKK+19-sever}. \nnew{Finally, as we show
in Appendix~\ref{app:list-reduction}, there is a simple and efficient post-processing algorithm that outputs a list of size $O(1/\alpha)$ without affecting the runtime or error guarantee by more than a constant factor.}

\paragraph{ Application to Learning Mixture Models}
As observed in~\cite{CSV17}, list-decoding generalizes the problem of learning mixtures.
Specifically, a list-decodable mean algorithm for bounded covariance distributions
can be used in a black-box manner (by treating a single cluster as the set of inliers)
to obtain an accurate clustering for mixtures of bounded covariance distributions.
If each distribution in the mixture has unknown covariance bounded by $\sigma^2 I$,
and the means of the components are separated by $\tilde{\Omega}(\sigma/\sqrt{\alpha})$,
we can perform accurate clustering, even in the presence of a small fraction of adversarial outliers.
This implication was shown in~\cite{CSV17}. Our new algorithm for list-decodable mean estimation
gives a simpler and faster method for this problem.

% Do we want to have sth in appendix? Do we want to have sth about stochastic optimization?

\paragraph {Technical Overview}
Here we describe our techniques
in tandem with a comparison to prior work.

The ``filtering'' framework~\cite{DKKLMS16, DKK+17}
works by iteratively detecting and removing outliers until the empirical
variance in every direction is not much larger than expected.
If every direction has small empirical variance,
this certifies that the the empirical mean is close to target mean.
Otherwise, a filtering algorithm projects the points in a direction
of large variance and removes (or reduces the weight of) those points
whose projections lie unexpectedly far from the empirical
median in this direction. In the small $\alpha$ setting, the one-dimensional ``outlier removal''
procedure is necessarily more complicated. For example, the input distribution can
simulate a mixture of $1/\alpha$ many Gaussians whose means are far from each other,
and the algorithm will have no way of knowing which is the real one.
To address this issue, one requires a more elaborate method, which we call a
{\em multifilter}. A multifilter can return several (potentially overlapping)
subsets of the original dataset with the guarantee that {\em at least one}
of these subsets is substantially ``cleaner''. This idea was introduced
in~\cite{DiakonikolasKS18-mixtures}, who gave a multifilter
for identity covariance Gaussians withe error $\tilde{O}(\alpha^{-1/2})$.
The multifilter of~\cite{DiakonikolasKS18-mixtures} makes essential use of the fact
that the covariance of the inliers is known and that the Gaussian distribution has very strong concentration. In this work, we build on~\cite{DiakonikolasKS18-mixtures} to develop
a multifilter for bounded covariance distributions.

We start by describing the Gaussian multifilter~\cite{DiakonikolasKS18-mixtures}.
Suppose we have found a large variance direction.
After we project the data in such a direction, there are two cases to consider.
The first is when almost all of the samples lie in some relatively short interval $I$.
In this case, the target mean must lie in that interval (as
otherwise an approximately $\alpha/2$ fraction of the good samples must
lie outside of this interval), and then samples that lie too far from
this interval $I$ are almost certainly outliers. The other case is more complicated.
If $\alpha < 1/2$, there might be multiple clusters of points which contain an $\alpha$ fraction of the
samples and could reasonably contain the inliers. If some pair of
these clusters lie far from each other, we might not be able to reduce
the variance in this direction simply by removing obvious outliers. In
this case,~\cite{DiakonikolasKS18-mixtures} find a pair of overlapping intervals
$I_1$ and $I_2$ such that with high probability either almost all the inliers
lie in $I_1$ or almost all the inliers lie in $I_2$. The algorithm then
recurses on both $I_1$ and $I_2$. To ensure that the complexity of
the algorithm does not blow-up with the recursion, \cite{DiakonikolasKS18-mixtures}
require that the sum of the squares of the numbers of remaining samples
in each subinterval is at most the square of the original number of samples.

At a high-level, our algorithm follows the same framework.
However, there were several key places where~\cite{DiakonikolasKS18-mixtures}
used the strong concentration bounds of the Gaussian assumption
that we cannot use in our context. For example, in the case where most of the samples are contained
within an interval $I$, Gaussian concentration bounds imply that
almost all of the good samples lie within distance $O(\sqrt{\log(1/\alpha)})$ of the
interval $I$, and therefore that almost all samples outside of this range
will be outliers. This is of course not true for heavy-tailed data.
To address this issue, we employ a soft-outlier procedure that reduces the weight
of each point based on its squared distance from $I$. The analysis in this case is much more
subtle than in the Gaussian setting.

The other more serious issue comes from the multi-filter case. With
Gaussian tails, so long as the subintervals $I_1$ and $I_2$ overlap for a distance of
$O(\sqrt{\log(1/\alpha)})$, this suffices to guarantee that the correct
choice of interval only throws away a $\poly(\alpha)$-fraction of the good
points. As long as at least an $\alpha$-fraction of the total points are
being removed, it is easy to see that this is sufficient. From
there it is relatively easy to show that, unless almost all
of the points are contained in some small interval, some
appropriate subintervals $I_1$ and $I_2$ can be found.
For bounded covariance distributions, our generalization of this case
is more complicated. In order to ensure that the fraction of good samples lost is small,
even if the true mean is exactly in the middle of the overlap between $I_1$ and $I_2$,
we might need to make this overlap quite large. In particular, in contrast to the Gaussian case,
we cannot afford to ensure that some small $\poly(\alpha)$ fraction of the inliers are lost.
In fact, we will need to adapt the fraction of inliers we
are willing to lose to the number of total points lost and ensure that
the fraction of inliers removed is {\em substantially better} than the
fraction of outliers removed (namely, by a $\log(1/\alpha)$ factor).
\nnew{This step is necessary for our new analysis of the behavior of the algorithm
under repeated applications of the multifilter.}
With this careful tuning, we can show that there will be an appropriate pair of intervals,
unless the distribution of points along the critical direction satisfy
inverse-quadratic tail bounds. This is not enough to show that there
is a short interval $I$ containing almost all of the points, but it will
turn out to be enough to show the existence of an $I$ containing almost
all of the points for which the variance of the points within $I$ is not
too large. This turns out to be sufficient for our analysis of the other case.

\paragraph{Concurrent and Independent Work.} Contemporaneous work~\cite{CMY20}, using different techniques, 
gave an algorithm for the same problem with asymptotic running time $\tilde{O}(n d / \alpha^c)$, 
for some (unspecified) constant $c$. At a high-level, the algorithm of \cite{CMY20} builds on the convex optimization frameworks
of~\cite{DKKLMS16, CDG18}, leveraging faster algorithms for solving structured SDPs.

\section{Preliminaries} \label{sec:prelims}

\paragraph{Notation.}
We write $\lg = \log_2$. For an interval $I = \closed{a}{b} = \closed{t-R}{t+R}$,
we will write $2I = \closed{t-2R}{t+2R}$.
\nnew{For a vector $v$, $\| v \|_2$ denotes its Euclidean norm.
For a symmetric matrix $M$, $\| M \|_2$ denotes its spectral norm.
We will use $\preceq$ to denote  the Loewner ordering between matrices,
i.e., for symmetric matrices $A, B$, we will write $A \preceq B$ to denote that
$B-A$ is positive semidefinite.}

For a set $T \subset \R^d$ we will often attach a weight function $w \colon T \to \closed{0}{1}$
and write $w(R) = \sum_{x\in R} w(x)$ for any subset $R \subseteq T$.
We will denote the weighted mean, weighted covariance matrix,
and weighted variance (in a given direction $v$) with respect to the weight function $w$
by $\mu_w(R) = \E_w[R] = \frac{1}{w(R)}\sum_{x\in R} w(x)x$,
$\Cov_w[R] = \frac{1}{w(R)} \sum_{x\in R} w(x)(x - \mu_w(R))(x - \mu_w(R))^T$,
and $\Var_w[v \cdot R] = \frac{1}{w(R)}\sum_{x\in R} w(x)(v \cdot x -  v \cdot \mu_w(R))^2$
for a subset $R \subseteq T$. \nnew{When the underlying weight function $w$ assigns the same weight
on each point, we will drop the index $w$ from these quantities. For example, we will use
$\mu(R)$ and $\Cov[R]$ for the empirical mean and covariance under the uniform distribution on the set $R$.}
Furthermore, we will write $\wPr{w}$ for the weighted probability
with respect to the weight function $w$ and $\Pr$ for the usual (counting) probability on sets.

\section{Algorithm and Analysis} \label{sec:alg}

\nnew{
In Section~\ref{sec:alg-setup}, we give a deterministic condition
under which our algorithm succeeds and bound the number of samples
needed to guarantee that this condition holds with high probability.
%In the subsequent sections, we give our algorithm and establish its correctness
%under these deterministic conditions.
In Section~\ref{ssec:basic-multifilter}, we present our basic multifilter.
In Section~\ref{ssec:main-alg}, we show how to use the basic multifilter
to obtain our list-decodable learning algorithm. In Section~\ref{ssec:runtime},
we analyze the running time of our algorithm. 
We conclude with some open problems in Section~\ref{sec:conc}.
}

%The multifilter at its base level requires a routine that given a polynomial
%$p$, where $p(T)$ behaves very differently from $p(G)$, allows us to use the
%values of $p(x)$ to separate the points coming from $G$ from the
%errors. The basic idea of the technique is to cluster the points $p(x)$,
%for $x \in T$, and throw away points that are too far from any cluster
%large enough to correspond to the bulk of the values of $p(G)$ (which
%must be well-concentrated), or to divide $T$ into two subsets with
%enough overlap to guarantee that any such cluster could be entirely
%contained on one side or the other. The details of this basic
%multifilter algorithm are covered in Section~\ref{ssec:basic-mf}.

%Given this basic multifilter, the high-level picture for our main subroutine is as
%follows: Using spectral methods we can find if there are any degree-$d$
%polynomials $p$ where $\E[p(T)^2]$ is substantially larger than it should
%be if $T$ consisted of samples from $N(\mu_T,I)$. If there are no such
%polynomials, it is not hard to see that $\|\mu-\mu_T\|_2$ is small giving us
%our desired approximation. Otherwise, we would like to apply our basic
%multifilter algorithm to get a refined version of $T$. The details of
%this routine can be found in Section~\ref{ssec:combining}.

\subsection{Setup and Main Theorem} \label{sec:alg-setup}
\nnew{
We define the following deterministic condition on the set of clean samples.}
%that guarantees that our algorithm will succeed on \emph{any} corrupted set $T$.
%A set of points $S \subset \R^d$ satisfying this deterministic condition
%will be called {\em representative}.

\begin{definition}[Representative set] \label{def:rep}
Let $D$ be a distribution on $\R^d$ with mean $\mu$ and covariance $\Sigma \preceq I$.
A set $S\subset \R^d$ is {\em representative} (with respect to $D$)
if $\norm{\Cov[S]}_2 \leq 1$ and $\norm{\mu(S) - \mu}_2 \leq 1$.
% Furthermore, for a weighted set $S_w \subset \R^d$, we say that $S_w$ is 2-representative (with respect to $D$) if $\norm{\Cov_w(S)} \leq 2$ and $\norm{\mu(S) - \mu}_2 \leq O(1)$.
\end{definition}

% \begin{lemma}
%   If $S$ is a representative set and $S_w$ is a weighted version of $S$ with $w(S_w) \geq \abs{S}/2$, then $S_w$ is 2-representative (with worse constant in the mean estimate).
% \end{lemma}

\nnew{
Our algorithm requires the following notion of goodness for the corrupted set $T$.}

%At a high-level, throughout its execution,
%our algorithm will produce several different weight functions f
%the original corrupted set $T$ it starts with, during its process of soft outlier removal.

\begin{definition}[Good set]\label{def:good-set}
 Let $D$ be a distribution on $\R^d$ with mean $\mu$ and covariance $\Sigma \preceq I$,
 and let $0<\alpha<1/2$. A set $T \subset \R^d$ is said to be $\alpha$-good (with respect to $D$)
 if there exists $S \subseteq T$ which is representative (with respect to $D$)
 and satisfies $\abs{S} \geq \alpha\abs{T}$.
\end{definition}

\medskip

\nnew{
\noindent In the subsequent subsections, we prove the following theorem:

\begin{theorem}[Main Theorem]\label{thm:main-det}
Suppose that $T$ is $\alpha$-good with respect to a distribution $D$ on $\R^d$. Then the algorithm \textsc{List-Decode-Mean}
runs in time $\tilde{O}(|T|^2 d / \alpha^2)$ and outputs a list of $O(1/\alpha^2)$ hypothesis vectors
%such that with high probability
at least one of which has $\ell_2$-distance $O(\log(1/\alpha)/\sqrt{\alpha})$ from the mean of $D$.
\end{theorem}

\paragraph{Sample Complexity.}
The deterministic conditions of Definitions~\ref{def:rep} and~\ref{def:good-set}
hold with high probability if the set $T$ has size $n = |T|  = \Omega(d/\alpha)$.
Note that $T$ contains a subset $G$ of $\alpha n \geq d$ i.i.d. samples from the distribution $D$.
The following lemma shows that with high probability $G$ contains a subset $S$
such that $|S| \geq |G|/2$ that satisfies the properties of a representative set, up to rescaling.

\begin{lemma}[see, e.g., Proposition 1.1 in~\cite{CSV17}]\label{lem:sample-complexity}
Let $D$ be a distribution on $\R^d$ with covariance matrix $\Sigma \preceq \sigma^2 I$, $\sigma>0$,
and $G$ be a multiset of $n \geq d$ i.i.d. samples from $D$. Then, with high probability, there exists a subset
$S \subseteq G$ of size $|S| \geq |G|/2$ such that $\norm{\Cov[S]}_2 \leq c \ \sigma^2$ and
$\norm{\mu(S) - \mu}_2 \leq c \ \sigma$, where $c>1$ is a universal constant independent of $D$.
\end{lemma}

We henceforth condition on the conclusions of Lemma~\ref{lem:sample-complexity} holding.
Note that by dividing each of our samples by $c \ \sigma$ we obtain a representative
set $S$ with respect to the distribution $(1/(c \sigma))D$.
Also note that the corrupted set $T$ will be $\alpha/2$-good.
By Theorem~\ref{thm:main-det}, we thus obtain a list of hypothesis one of which has $\ell_2$-error
$O(\log(1/\alpha)/\sqrt{\alpha})$. By rescaling back, we get an estimate of the true mean $\mu$ of $D$
within $\ell_2$ error $O(\sigma \log(1/\alpha)/\sqrt{\alpha})$, as desired. This proves Theorem~\ref{thm:main-intro}.
}

\medskip

\nnew{Throughout this section, we will denote by $T$ the initial (corrupted) set of points
and by $S \subset T$ a representative set with $|S| \geq \alpha|T|$.}

\subsection{Basic Multifilter} \label{ssec:basic-multifilter}

\nnew{
The basic multifilter is a key subroutine of our algorithm. Intuitively, it takes
as input a large variance direction and, under certain assumptions,
splits the dataset into at most two (overlapping) datasets at least one of which
is cleaner. Since we are employing a soft outlier removal procedure,
the real version of the routine starts from a weight function on the dataset $T$
and produces one or two weight functions on $T$ with desirable properties.
%Our basic multifilter is presented in pseudocode below.
}

\fbox{%
  \begin{minipage}{0.95 \linewidth}
    \textbf{Algorithm} \textsc{BasicMultifilter}

    Input: unit vector $v\in \R^d$, $T\subset \R^d$ and weight function $w$ on $T$, $0<\alpha<1/2$

    \begin{enumerate}[1.]
    \item Let $C>0$ be a sufficiently large universal constant.

  \item \label{step:bm-quantile} Let $a \in \R$ be maximal such that $w(\set{x \in T : v \cdot x < a}) \leq \alpha w(T)/8$ and $b$ be minimal such that $w(\set{x \in T : v \cdot x > b}) \leq \alpha w(T)/8$.
    Let $I = \closed{a}{b}$.

%    \item \label{step:bm-quantile} Let $a \in \R$ be maximal such that $w(\set{x \in T : v \cdot x < a}) \leq \alpha w(T)/8$
%    if such an $a$ exists; otherwise let $a$ be the minimal point of $\set{v \cdot x : x \in T}$.
%    Similarly, let $b$ be minimal such that $w(\set{x \in T : v \cdot x > b}) \leq \alpha w(T)/8$
%    if such $b$ exists; otherwise let $b$ be the maximal point of $\set{v \cdot x : x \in T}$.
%    Let $I = \closed{a}{b}$.

    \item \label{step:bm-if} If  $\Var_w[v\cdot (T \cap 2I)] \leq C\cdot \log(2/\alpha)^2$, then
      \begin{enumerate}[(a)]
      \item If $\Var_w[v\cdot T] \leq 2C \cdot \log(2/\alpha)^2$, return ``YES''.

      \item Let $f(x) = \min_{t\in \closed{a}{b}} \abs{v\cdot x - t}^2$,
      and redefine the weight of each $x\in T$ by
        \begin{equation*}
          w\new(x) = \Bigl(1-\frac{f(x)}{\max_{x\in T}f(x)}\Bigr) w(x).
        \end{equation*}
      \item Return $\set{(T,w\new,\alpha)}$.
      \end{enumerate}

    \item \label{step:bm-two-weights} If $I$ does not satisfy the condition of Step~\ref{step:bm-if},
    then
      \begin{enumerate}[(a)]
      \item Find $t \in \R$ and $R>0$ such that the sets
      $T_1 = \set{x\in T : v\cdot x \geq t-R}$ and $T_2 = \set{x\in T : v\cdot x < t+R}$ satisfy
        \begin{equation}\label{eq:Tineq}
          w(T_1)^2 + w(T_2)^2 \leq w(T)^2 \;,
        \end{equation}
        and
        \begin{equation}\label{eq:minineq}
          \min\Bigl( 1-\frac{w(T_1)}{w(T)}, 1-\frac{w(T_2)}{w(T)} \Bigr) \geq \frac{48\lg(2/\alpha)}{R^2} \;.
        \end{equation}
        Define two weight functions $w^{(1)}$ and $w^{(2)}$ on $T$ by multiplying
        the indicator functions of $T_1$ and $T_2$ with the weight function $w$.
      \item Return $\set{(T,w^{(1)},\alpha), (T,w^{(2)},\alpha)}$.
      \end{enumerate}
    \end{enumerate}
    \smallskip
  \end{minipage}%
}

\medskip

In the body of this subsection, we show that the \textsc{BasicMultifilter} algorithm
has certain desirable properties that we will later use to establish correctness of our
main algorithm.

The following notation will facilitate our analysis.
We will denote  $\Delta w(S) = w(S) - w\new(S)$ and
$\Delta w(T) = w(T) - w\new(T)$
to describe the change of weights during a step of the \textsc{BasicMultifilter} algorithm.

\nnew{Our first lemma bounds the relative change in the weight of $S$ and $T$ if
the \textsc{BasicMultifilter} algorithm outputs a single weight function $w\new$ in
Step~\ref{step:bm-if}(c).}

\begin{lemma}\label{lem:T'ineq}
If $T$ is $\alpha$-good and $w(S) \geq 3\abs{S}/4$,
then after Step~\ref{step:bm-if}(b) of \textsc{BasicMultifilter} we have
%\begin{equation*}
$\frac{\Delta w(S)}{w(S)} \leq \frac{\Delta w(T)}{w(T)} \cdot \frac{1}{24\lg(2/\alpha)}.$
%\end{equation*}
\end{lemma}

Before we proceed with the formal proof of Lemma~\ref{lem:T'ineq},
we provide an overview of the argument (without the constant of $24$) for the sake of intuition.

Firstly, we note that $v\cdot \nnew{\mu_w(S)} \in [a-O(1),b+O(1)]$.
This is because if, say, $\mu_w(S)$ was much less than $a$, then since all but a small fraction
of the points in $S$ have $v\cdot (x-\mu_w(S)) = O(1)$, this would imply that most of the points of $S$ are less than $a$.
But since all but a $1/4$-fraction of the points in $S$ remain under weight $w$,
and since they account for at least an $\alpha$ fraction of the weight of $T$,
this would imply that more than an $\alpha/8$-fraction of the weight of $T$ was less than $a$,
which is a contradiction.

Given this, we have that \nnew{$f(x) = O(1+(v\cdot(x-\mu_w(S)))^2)$},
and therefore the average value of $f$ over $S$ is $O(1)$. On the other hand,
\[
\Var_w[v\cdot T] \leq \Var_w[v\cdot T\cap 2I]+O(\E_w[f(T)]) \;.
\]
This implies that since $\Var_w[v\cdot T]$ is large, $\E_w[f(T)]$ is $\Omega(\log^2(1/\alpha))$.

Finally, since we are downweighting point $x$ by an amount proportional to $f(x)$,
it is easy to see that $\Delta w(T)/w(T)$ is proportional to $\E_w[f(T)]$,
while $\Delta w(S)/w(S)$ is proportional to $\E_w[f(S)]$, and the lemma follows.

\begin{proof}[of Lemma~\ref{lem:T'ineq}]
Since $S$ is representative and $w(S) \geq 3\abs{S}/4 \geq \abs{S}/2$, we see that
\[
\Var_w[v \cdot S] \leq \frac{1}{w(S)} \littlesum_{x \in S} (v \cdot x - v \cdot \mu(S))^2 \leq
\frac{2}{\abs{S}} \littlesum_{x \in S} (v \cdot x - v \cdot \mu(S))^2 = 2\Var[v \cdot S] \leq 2 \;.
\]
If $v \cdot \mu_w(S) \notin \closed{a -2}{b + 2}$,
then
\[
\Var_w[v \cdot S] \geq \frac{1}{w(S)} \littlesum_{\substack{x \in S \\ v \cdot x \in \closed{a}{b}}} w(x)(v \cdot x - v \cdot \mu_w(S))^2
> \frac{1}{w(S)} \littlesum_{\substack{x \in S \\ v \cdot x \in \closed{a}{b}}} 4w(x) \geq \frac{4}{2} = 2 \;,
\]
since $w(\set{x \in S : v \cdot x \in \closed{a}{b}}) \geq w(S) - \alpha w(T)/4 \geq 3\abs{S}/4 - \alpha\abs{T}/4
\geq \abs{S}/2 \geq w(S)/2$ (since $\alpha w(T)/4 \leq \alpha\abs{T}/4 \leq \abs{S}/4$ because $T$ is $\alpha$-good),
a contradiction.
Hence, we have that $v\cdot \mu_w(S) \in \closed{a-O(1)}{b+O(1)}$.

We note that if
\[
\littlesum_{x\in S} w(x)f(x) \geq \frac{w(S)}{24w(T)\lg(2/\alpha)}\littlesum_{x\in T} w(x)f(x) \;,
\]
then
\[
\littlesum_{x\in T} w(x)f(x) \leq \frac{24w(T)\lg(2/\alpha)}{w(S)}\littlesum_{x\in S} w(x)f(x) \;,
\]
where
\[
\littlesum_{x\in S} w(x)f(x) \leq
\littlesum_{x\in S} w(x)\bigl( (v\cdot x - v\cdot \mu_w(S)) + O(1) \bigr)^2
= O\Bigl( w(S)\Var_w[v\cdot S] + w(S) \Bigr) \leq O(w(S)) \;.
\]
Thus, we can write
  \begin{align*}
    \Var_w[v\cdot T] &\leq \frac{1}{w(T)}\littlesum_{x\in T} w(x)\bigl(v\cdot x - \E_w[v\cdot T \cap 2I]\bigr)^2 \\
                 &\leq \frac{1}{w\bigl( \set{x \in T : v\cdot x \in 2I} \bigr)}\littlesum_{\substack{x\in T \\ v\cdot x \in 2I}} w(x)\bigl(v\cdot x - \E_w[v \cdot T \cap 2I]\bigr)^2 \\*
                 &\phantom{{}={}} + \frac{1}{w(T)}\littlesum_{\substack{x\in T \\ v\cdot x \notin 2I}} w(x)\bigl(v\cdot x - \E_w[v \cdot T \cap 2I]\bigr)^2 \\
                 &\leq \Var_w[v\cdot T \cap 2I] + \frac{1}{w(T)}\littlesum_{\substack{x\in T \\ v\cdot x \notin 2I}} w(x)O(f(x)) \\
                 &\leq C\log(2/\alpha)^2 + O\Bigl(\frac{1}{w(T)}\littlesum_{x\in T} w(x)f(x) \Bigr) \\
                 &\leq C\log(2/\alpha)^2 + O\Bigl(\frac{24\lg(2/\alpha)}{w(S)}\littlesum_{x\in S} w(x)f(x) \Bigr) \\
                 &\leq C\log(2/\alpha)^2 + O\Bigl( \frac{24\lg(2/\alpha)O(w(S))}{w(S)} \Bigr) \\
                 &= C\log(2/\alpha)^2 + O(\log(2/\alpha)) \leq C\log(2/\alpha)^2 + O(\log(2/\alpha)^2) \;,
  \end{align*}
  which is a contradiction, as this is the condition of Step~\ref{step:bm-if}(a).

Hence, we have shown that
\[
\littlesum_{x\in S} w(x)f(x) \leq \frac{w(S)}{24w(T)\lg(2/\alpha)}\littlesum_{x\in T} w(x)f(x) \;,
\]
  and thus
\[
    w(x)-w\new(x) = \frac{f(x)}{\max_{x\in T}f(x)}w(x) \;,
\]
  implies that
  \begin{equation*}
    \frac{\Delta w(S)}{\Delta w(T)} = \frac{\littlesum_{x\in S}(w(x)-w\new(x))}{\littlesum_{x\in T}(w(x)-w\new(x))} = \frac{\littlesum_{x\in S} w(x)f(x)}{\littlesum_{x\in T} w(x)f(x)} \leq \frac{w(S)}{24w(T)\lg(2/\alpha)} \;,
  \end{equation*}
  which completes the proof of Lemma~\ref{lem:T'ineq}.
\end{proof}

\nnew{Our second lemma says that conditions \eqref{eq:Tineq} and \eqref{eq:minineq} in
Step~\ref{step:bm-two-weights}(a) of the algorithm are satisfiable.}

\begin{lemma}\label{lem:bm-sat}
If \textsc{BasicMultifilter} reaches Step~\ref{step:bm-two-weights}(a),
there exist $t\in \R$ and $R>0$ such that the conditions \eqref{eq:Tineq} and \eqref{eq:minineq} are satisfied.
\end{lemma}
\begin{proof}
For $t\in \R$ and $R>0$, we will use the notation $g(t+R) = 1-\frac{w(T_2)}{w(T)}$ and
$g^c(t-R)= 1-\frac{w(T_1)}{w(T)}$ to describe the tails \nnew{of the weight distribution}.
%  \begin{align*}
%    g(t+R) &= 1-\frac{w(T_2)}{w(T)} = \frac{w(\set{x\in T : v\cdot x \geq t+R})}{w(T)}, \\
%    g^c(t-R) &= 1-\frac{w(T_1)}{w(T)} = \frac{w(\set{x\in T : v\cdot x < t-R})}{w(T)}
%  \end{align*}
Thus, \eqref{eq:Tineq} and \eqref{eq:minineq} become
  \begin{equation}\label{eq:gineq1}
    (1-g^c(t-R))^2 + (1-g(t+R))^2 \leq 1
  \end{equation}
  and
  \begin{equation}\label{eq:gineq2}
    \min(g^c(t-R),g(t+R)) \geq 48\lg(2/\alpha)/R^2 \;.
  \end{equation}
  Now assume for contradiction that we cannot find any $t\in \R$ and $R>0$
  satisfying both \eqref{eq:gineq1} and \eqref{eq:gineq2}, i.e., either \eqref{eq:gineq1} fails or \eqref{eq:gineq2} fails.
%We will show that $g(t+R)$ cannot be too big, and by a similar argument
%it follows that $g^c(t-R)$ cannot be too big.
  Let $\mathrm{med} = \median_w(v\cdot T)$.

  Let $x=x_0>\mathrm{med}$ and let $\gamma = \gamma_0 = g(x_0)$,
  and note that $\gamma_0 \leq 1/2$. We want to show that
  \begin{equation}\label{eq:xineq}
    x \leq \mathrm{med} + O\Bigl( \sqrt{\lg(2/\alpha)/\gamma} \Bigr) \;.
  \end{equation}
  First find $t_0$ and $R_0$ such that $x_0=t_0+R_0$ and $\gamma_0 = 48\lg(2/\alpha)/R_0^2$, i.e.,
$R_0 = \sqrt{48\lg(2/\alpha)/\gamma_0}$.
  Then either $t_0-R_0 \leq \mathrm{med}$ or $t_0-R_0 > \mathrm{med}$.
  If $t_0-R_0 \leq \mathrm{med}$, then $x = t_0+R_0 \leq \mathrm{med} + 2R_0$
  and we indeed get \eqref{eq:xineq}. On the other hand, if $t_0-R_0 > \mathrm{med}$
  we see that $g^c(t_0-R_0)\geq 1/2 \geq \gamma_0$, so \eqref{eq:gineq2} is satisfied.
  Thus, \eqref{eq:gineq1} must fail (by assumption), i.e.,
  %\begin{equation*}
    $g(t_0-R_0)^2 + (1-\gamma_0)^2 > 1,$
  %\end{equation*}
  since $g(t_0-R_0) = 1 - g^c(t_0-R_0)$. So
  \begin{equation*}
    g(t_0-R_0)^2 > 1 - (1-\gamma_0)^2 = 2\gamma_0 - \gamma_0^2 = \gamma_0 + (\gamma_0-\gamma_0^2) > \gamma_0,
  \end{equation*}
  and thus
  \begin{equation*}
    g(x-2R_0) = g(x_0-2R_0) = g(t_0-R_0) > \sqrt{\gamma_0} = \gamma^{1/2}.
  \end{equation*}
  Now let $x_1 = x_0-2R_0 > \mathrm{med}$ and let $\gamma_1 = g(x_1) \leq 1/2$.
  Note that $\gamma_1 > \sqrt{\gamma_0}$.
  By finding $t_1$ and $R_1$ as before and following the same argument, we get that
  \begin{equation*}
    g(x-2R_0-2R_1) = g(x_1 - 2R_1) = g(t_1-R_1) > \sqrt{\gamma_1} > \gamma^{1/2^2} \;.
  \end{equation*}
  Continuing like this, we inductively  get that
  \begin{equation*}
    g(x_n) = g\Bigl(x-2\littlesum_{i=0}^{n-1} R_i\Bigr) > \sqrt{\gamma_{n-1}} > \gamma^{1/2^n} \;,
  \end{equation*}
  as long as $x_{n-1} = x-2\littlesum_{i=0}^{n-2}R_i > \mathrm{med}$. Hence
  %\begin{equation*}
    $\gamma_n = g(x_n) > \gamma^{1/2^n},$
  %\end{equation*}
  and thus $x_{\lg\lg(1/\gamma)} < \mathrm{med}$ since $g(x_{\lg\lg(1/\gamma)})>1/2$.
  Therefore
  \begin{align*}
    \mathrm{med} > x_{\lg\lg(1/\gamma)} &= x - 2\littlesum_{i=0}^{\lg\lg(1/\gamma)-1}R_i = x - 2\sqrt{48\lg(2/\alpha)}\littlesum_{i=0}^{\lg\lg(1/\gamma)-1} \frac{1}{\sqrt{\gamma_i}} \\
                        &\geq x - O\Bigl( \sqrt{\lg(1/\alpha)} \littlesum_{i=1}^{\lg\lg(1/\gamma)}\frac{1}{\gamma^{1/2^{i}}} \Bigr) \geq x - O\Bigl( \frac{\sqrt{\lg(2/\alpha)}}{\sqrt{\gamma}} \Bigr),
  \end{align*}
  i.e.,
  %\begin{equation*}
    $x \leq \mathrm{med} + O\Bigl(\sqrt{\lg(2/\alpha)/\gamma} \Bigr).$
  %\end{equation*}
  Now writing $\gamma = g(\mathrm{med}+t)$, for $t>0$, the above gives that
$t \leq O\Bigl(\sqrt{\lg(2/\alpha)/\gamma}\Bigr)$,
  and thus
  \begin{equation*}
    \wPr{w}_{y\in T} [v\cdot y > \mathrm{med} + t] = O(g(\mathrm{med}+t)) = O(\gamma) \leq O\Bigl(\lg(2/\alpha)/t^2 \Bigr) \;.
  \end{equation*}
  \nnew{A very similar proof yields the analogous result} for $g^c(m-t)$, so
  %\begin{equation*}
    $\wPr{w}_{y\in T}[\abs{v\cdot y - \mathrm{med}} > t] \leq O\Bigl(\lg(2/\alpha)/t^2 \Bigr).$
  %\end{equation*}
  Letting $a$ and $b$ be as in Step~\ref{step:bm-quantile} of \textsc{BasicMultifilter}, we note that
  \begin{equation*}
    g(b-1) = w\bigl( \set{x\in T : v \cdot x \geq b-1} \bigr)/w(T) \geq \alpha/8
  \end{equation*}
  by the definition of $b$, so
  \begin{equation*}
    b-1 \leq \mathrm{med} + O\Bigl( \sqrt{\lg(2/\alpha)}/\sqrt{\alpha/4} \Bigr) \leq \mathrm{med}
    + O( 1/\alpha ) \;,
  \end{equation*}
  and thus $b \leq \mathrm{med} + O(1/\alpha)$.
  \nnew{An analogous argument yields} a similar result for $g^c(a)$,
  so $2I \subset \closed{\mathrm{med}-O(1/\alpha)}{\mathrm{med}+O(1/\alpha)}$.

  Finally we note that $w(\set{y \in T : v \cdot y \notin 2I}) \leq \alpha w(T)/4 \leq w(T)/2$, so
  \begin{equation*}
    w\bigl( \set{y \in T : v \cdot y \in 2I} \bigr) = w(T) - w\bigl( \set{y \in T : v \cdot y \notin 2I} \bigr) \geq w(T) - w(T)/2 = w(T)/2 \;,
  \end{equation*}
  and thus
  \begin{align*}
    \wPr{w}_{y \in \set{z \in T : v \cdot z \in 2I}}[\abs{v \cdot y - \mathrm{med}} > t]
    &= \frac{w\bigl( \set{y \in T : \abs{v \cdot y - \mathrm{med}} > t \text{ and } v \cdot y \in 2I} \bigr)}{w\bigl( \set{y \in T : v \cdot y \in 2I} \bigr)} \\
    &\leq \frac{w\bigl( \set{y \in T : \abs{v \cdot y - \mathrm{med}} > t} \bigr)}{w(T)/2} \\
    &= 2\wPr{w}_{y\in T}[\abs{v \cdot y - \mathrm{med}} > t] \;.
  \end{align*}
  Hence, we have that
  \begin{align*}
    \Var_w[v\cdot T\cap 2I] &\leq 2\int_0^{O(1/\alpha)} 2t \cdot \wPr{w}_{y\in T}[\abs{v\cdot y - \mathrm{med}}>t]dt
    \leq O(\lg(2/\alpha))\int_1^{O(1/\alpha)} (1/t) dt \\
                     &= O(\log(2/\alpha)^2).
  \end{align*}
Thus, if conditions \eqref{eq:Tineq} and \eqref{eq:minineq} were not satisfiable,
the condition of Step~\ref{step:bm-if} in \textsc{BasicMultifilter} would have been satisfied.
This is a contradiction and completes the proof of Lemma~\ref{lem:bm-sat}.
\end{proof}

\nnew{Our next lemma bounds the relative change in the weight of $S$ and $T$ if
the \textsc{BasicMultifilter} algorithm outputs two weight functions
in Step~\ref{step:bm-two-weights}(b).}

\begin{lemma}\label{lem:Tiineq}
If $T$ is $\alpha$-good and $w(S) \geq 3\abs{S}/4$,
then after Step~\ref{step:bm-two-weights}(b) of \textsc{BasicMultifilter}
we have that one of $w^{(1)}$ and $w^{(2)}$ will satisfy
  %\begin{equation*}
    $\frac{\Delta^{(i)} w(S)}{w(S)} \leq \frac{\Delta^{(i)} w(T)}{w(T)} \cdot \frac{1}{24\lg(2/\alpha)} \;,$
  %\end{equation*}
  where $\Delta^{(i)}w = w - w^{(i)}$ for $i=1,2$.
\end{lemma}

\begin{proof}[of Lemma~\ref{lem:Tiineq}]
Recall from the proof of Lemma~\ref{lem:T'ineq} that $\Var_w[v \cdot S] \leq 2$.
So choosing $i\in \set{1,2}$ such that
  \begin{equation*}
    v\cdot T \cap \open{v\cdot \mu_w(S) - R}{v\cdot \mu_w(S) + R} \subseteq v\cdot T_i \;,
  \end{equation*}
  where the $T_i$ are as in Step~\ref{step:bm-two-weights}(a), we get that
  \begin{equation*}
    \frac{\Delta^{(i)} w(S)}{w(S)} = \wPr{w}_{x\in S}[x\notin T_i] \leq \wPr{w}_{x\in S}[\abs{v\cdot x-v\cdot \mu_w(S)}>R] \leq \frac{\Var_w[v\cdot S]}{R^2} \leq \frac{2}{R^2} \;,
  \end{equation*}
  and thus
  %\begin{equation*}
    $\Delta^{(i)} w(S) \leq 2 w(S)/R^2$
  %\end{equation*}
  for one of $i=1,2$. For this $i$, we also have
  \begin{equation*}
    \Delta^{(i)} w(T) = w(T)\bigl( 1 - w^{(i)}(T)/w(T) \bigr) \geq w(T) 48\lg(2/\alpha)/R^2 \;,
  \end{equation*}
  so
  \begin{equation*}
    \frac{\Delta^{(i)} w(S)}{\Delta^{(i)} w(T)} \leq \frac{w(S)}{w(T)} \cdot \frac{1}{24\lg(2/\alpha)} \;,
  \end{equation*}
  which is equivalent to the claim Lemma~\ref{lem:Tiineq}.
\end{proof}

\nnew{Combining Lemmas~\ref{lem:T'ineq} and~\ref{lem:Tiineq}, we obtain the following corollary.}

\begin{corollary}\label{cor:stepineq}
If $T$ is $\alpha$-good and $w(S) \geq 3\abs{S}/4$, then in each \nnew{iteration}
of \textsc{BasicMultifilter} returning new weight functions, for at least one of the new weight functions returned,
we have that
\begin{equation}\label{eq:stepineq}
\frac{\Delta w(S)}{w(S)} \leq \frac{\Delta w(T)}{w(T)} \frac{1}{24\lg(2/\alpha)} \;.
\end{equation}
\end{corollary}

\noindent The following definition facilitates the analysis in the next subsection.

\begin{definition}[Nice \nnew{iteration}]
We will call an \nnew{iteration} of \textsc{BasicMultifilter} from the old weight function
$w$ to the new weight function $w'$ such that \eqref{eq:stepineq} is satisfied
a nice \nnew{iteration}.
\end{definition}

\subsection{Main Algorithm} \label{ssec:main-alg}

\nnew{Our main algorithm is presented in pseudocode below.}

\medskip

\fbox{%
  \begin{minipage}{0.95\linewidth}
    \textbf{Algorithm} \textsc{MainSubroutine}

    Input: $T\subset \R^d$ and weight function $w$ on $T$, $0<\alpha<1/2$

    \begin{enumerate}[1.]
      \item Let $\Sigma_{T,w} = \Cov_w[T]$ be the weighted covariance matrix.

      \item Let $\lambda$ be the top eigenvalue and $v$ an associated unit eigenvector of $\Sigma_{T,w}$.
      Compute approximations $\lambda^*$ and $v^*$ to these satisfying $(v^*)^T\Sigma_{T,w}v^* = \lambda^*$ and $\lambda \geq \lambda^* \geq \lambda/2$.

      \item Run \textsc{BasicMultifilter}$(v^*,T,w,\alpha)$.

        \begin{enumerate}[(a)]
        \item If it returns ``YES'', then return $\mu_w(T)$.

        \item If it returns a list $\set{(T,w',\alpha)}$, then return \nnew{the list containing the elements of
        $\set{(T,w',\alpha)}$} with $w'(T) \geq \alpha\abs{T}/2$.
        \end{enumerate}
    \end{enumerate}
  \end{minipage}%
}

\medskip

\fbox{%
  \begin{minipage}{0.95\linewidth}
    \textbf{Algorithm} \textsc{List-Decode-Mean}

    Input: $T\subset \R^d$, $0<\alpha<1/2$

    \begin{enumerate}[1.]
    \item Let $L = \set{(T,w^{(0)},\alpha)}$, where $w^{(0)}(x) = 1$ for all $x \in T$, and let $M=\emptyset$.

    \item While $L\neq \emptyset$:
      \begin{enumerate}[label=(\alph*)]
      \item Get the first element $(T,w,\alpha)$ from $L$ and remove it from the list.

      \item Run \textsc{MainSubroutine}$(T,w,\alpha)$.
        \begin{enumerate}[(a)]
        \item If this routine returns a vector, then add it to $M$.
        \item If it returns a list of $(T,w',\alpha)$, append that to $L$.
        \end{enumerate}
      \end{enumerate}

    \item Output $M$ as a list of guesses for the target mean $\mu$ of $D$.
    \end{enumerate}
  \end{minipage}%
}

\bigskip

\nnew{Our first lemma of this section establishes that, under certain conditions, if
\textsc{MainSubroutine} returns a hypothesis vector, this vector will be close to the target mean.}

\begin{lemma} \label{lem:vector-close}
If $T$ is $\alpha$-good, $w(S) \geq 3\abs{S}/4$, and \textsc{MainSubroutine} returns a vector $\mu_w(T)$,
then we have that $\norm{\mu-\mu_w(T)}_2 \leq O\Bigl(\log(1/\alpha)/\sqrt{\alpha}  \Bigr)$.
\end{lemma}

We start with an intuitive overview of the proof.
If $\beta=w(S)/w(T)$, then for any unit vector $v$, we have that
\[
\Var_w[v\cdot T] \geq \beta(v\cdot(\mu_w(S)-\mu_w(T)))^2 \;.
\]
Because the algorithm returned a vector at this step, we have that $\Var_w[v\cdot T] = O(\log^2(1/\alpha))$, and by our assumptions $\beta \gg \alpha$. Together these imply that $|v\cdot(\mu_w(S)-\mu_w(T))| = O(\log(1/\alpha)/\sqrt{\alpha})$.
Since this holds for all directions, $\|\mu_w(S)-\mu_w(T)\|_2 = O(\log(1/\alpha)/\sqrt{\alpha})$.
Finally, since we kept a constant fraction of the mass of $S$, and since the covariance of $S$ is $O(I)$,
a similar argument tells us that $\|\mu_w(S)-\mu\|_2=O(1)$.
Combining these with the triangle inequality gives the lemma.

We can now proceed with the formal proof.
We start with the following useful claim:

\begin{claim}\label{lem:lambdaineq}
Any unit vector $v$ has $\E_w[(v\cdot(T-\mu_w(T)))^2] \leq 2\lambda^*$ and $\E_w[(v^*\cdot(T-\mu_w(T)))^2] = \lambda^*$.
If \textsc{MainSubroutine}$(T,w,\alpha)$ returns a vector, then $\lambda^*=O(\log(1/\alpha)^2)$.
\end{claim}
\begin{proof}
We note that
$\E_w[(v\cdot(T-\mu_w(T)))^2] = \Var_w[v\cdot T] = v^T\Sigma_{T,w} v \leq \lambda \norm{v}_2^2 = \lambda \leq 2\lambda^*$.
In the case that $v = v^*$, we see that $(v^*)^T\Sigma_{T,w}v^* = \lambda^*$,
and when \textsc{MainSubroutine} returns a vector, then
$\Var_w[v^*\cdot T] = O(\log(1/\alpha)^2)$ by construction.
\end{proof}

Our formal proof below is slightly different than the one sketched above, but the main idea is similar.

\begin{proof}[of Lemma~\ref{lem:vector-close}]
If $T$ is $\alpha$-good, we have that $w(S) \ge 3\abs{S}/4 \geq \abs{S}/2 \geq \alpha\abs{T}/2 \geq \alpha w(T)/2$.
Therefore,
  \begin{align*}
    \E_w[(v\cdot(S - \mu_w(T)))^2] &= \frac{1}{w(S)}\littlesum_{x\in S}w(x)(v\cdot(x-\mu_w(T)))^2 \leq \frac{2}{\alpha w(T)}\littlesum_{x\in T}w(x)(v\cdot(x-\mu_w(T)))^2 \\
                           &= \frac{2\E_w[(v\cdot(T-\mu_w(T)))^2]}{\alpha} \leq \frac{4\lambda^*}{\alpha} \;,
  \end{align*}
  i.e.,
  \begin{equation}\label{eq:EwS}
    \E_w[(v\cdot(S - \mu_w(T)))^2] \leq 4\lambda^*/\alpha \;.
  \end{equation}
  Note that $\E[v\cdot(S-\mu_w(T))] = v\cdot(\mu(S)-\mu_w(T))$ and $\Var[v\cdot(S - \mu_w(T))] = \Var[v \cdot S]\leq 1$. So by Cantelli's inequality
  \begin{equation*}
    \Pr[v\cdot(S-\mu_w(T)) \geq v\cdot(\mu(S)-\mu_w(T)) - 1] \geq 1 - \frac{\Var[v\cdot(S-\mu_w(T))]}{\Var[v\cdot(S-\mu_w(T))]+1} \geq 1 - 1/2 = 1/2,
  \end{equation*}
  since $x/(x+1)$ is increasing on $\open{0}{\infty}$ so
  \begin{align*}
    \frac{\Var[v\cdot(S-\mu_w(T))]}{\Var[v\cdot(S-\mu_w(T))]+1} \leq 1/2 \;.
  \end{align*}
  Now $w(S) \geq 3\abs{S}/4$, so
  \begin{align*}
    \wPr{w}[v\cdot(S - \mu_w(T)) \geq v\cdot(\mu(S)-\mu_w(T))-1] &\geq 1/4.
  \end{align*}
  By Markov's inequality applied to \eqref{eq:EwS}, we get
  \begin{align*}
    \wPr{w}[v\cdot(S - \mu_w(T)) \geq \sqrt{16\lambda^*/\alpha}] \leq \frac{\E_w[(v\cdot(S - \mu_w(T)))^2]}{(\sqrt{16\lambda^*/\alpha})^2} \leq \frac{4\lambda^*/\alpha}{16\lambda^*/\alpha} = \frac{1}{4},
  \end{align*}
  and so
  \begin{align*}
    v\cdot(\mu(S)-\mu_w(T)) - 1 \leq \sqrt{16\lambda^*/\alpha}.
  \end{align*}
  Thus, if \textsc{MainSubroutine} returns a vector, we have by Claim~\ref{lem:lambdaineq} that
  \begin{align*}
    v\cdot(\mu(S)-\mu_w(T)) \leq 1 + \sqrt{16\lambda^*/\alpha} \leq O\Bigl( \log(2/\alpha)/\sqrt{\alpha} \Bigr)
  \end{align*}
  for all unit vectors $v$. Hence, we obtain
  \[
    \norm{\mu-\mu_w(T)}_2 \leq \norm{\mu-\mu(S)}_2 + \norm{\mu(S)-\mu_w(T)}_2 \leq 1 + O\Bigl( \log(2/\alpha)/\sqrt{\alpha} \Bigr) = O\Bigl( \log(1/\alpha)/\sqrt{\alpha} \Bigr) \;.
  \]
\end{proof}

\nnew{So far, we have shown that if the algorithm \textsc{List-Decode-Mean} reaches a stage
in which the \textsc{BasicMultifilter} routine returns the vector $\mu_w(T)$,
where the current weight function $w$ satisfies $w(S) \geq 3\abs{S}/4$,
then $\mu_w(T)$ is an accurate estimate of the target mean $\mu$.
It remains to show that \textsc{List-Decode-Mean} will provably reach such a stage.}

\begin{lemma}\label{lem:Tfinal}
Assume $T$ is $\alpha$-good. Then, following a sequence of nice \nnew{iterations} of \textsc{BasicMultifilter}
%from Corollary~\ref{cor:stepineq}
in \textsc{List-Decode-Mean} starting from \nnew{the uniform weight function} $w^{(0)}$,
we \nnew{obtain a weight function} $w$ with $w(S) \geq 3\abs{S}/4$ for which the \textsc{BasicMultifilter}
subroutine returns ``YES''.
\end{lemma}

\begin{proof}[of Lemma~\ref{lem:Tfinal}]
  Let $w^{(0)}$ be the weight function on $T$ given by $w^{(0)}(x) = 1$ for all $x\in T$
  so that $w^{(0)}(T) = \abs{T}$, and let
  $A_k = \left\{ w \colon T \to \closed{0}{1} \mid \abs{T}/2^k < w(T) \leq \abs{T}/2^{k-1} \right\}$
  for $k\geq 1$.

  Clearly, $w^{(0)}(S) = \abs{S} \geq 3\abs{S}/4$, so by Corollary \ref{cor:stepineq} the \textsc{BasicMultifilter} there is a nice first \nnew{iteration}.

  Suppose that we have a sequence
  \begin{equation*}
    \begin{tikzcd}[column sep=huge]
      w^{(i)} \ar[r,"\Delta^{(i+1)}w"] & w^{(i+1)} \ar[r,"\Delta^{(i+2)} w"] & \cdots  \ar[r,"\Delta^{(j)} w"] & w^{(j)}
    \end{tikzcd}
  \end{equation*}
  of nice  \nnew{iterations}, where $w^{(i)},w^{(i+1)},\dotsc,w^{(j)} \in A_k$ for some fixed $k$.
  Then, by Corollary~\ref{cor:stepineq} and the definition of $A_k$, we get that
  \begin{align*}
    \littlesum_{m=i}^{j-1} \Delta^{(m+1)} w(S) &\leq \littlesum_{m=i}^{j-1} \frac{w^{(m)}(S)}{w^{(m)}(T)} \cdot \frac{1}{24\lg(2/\alpha)}\Delta^{(m+1)} w(T) \leq \frac{\abs{S}}{\abs{T}}\frac{2^k}{24\lg(2/\alpha)}\littlesum_{m=i}^{j-1} \Delta^{(m+1)} w(T) \\
                               &\leq \frac{\abs{S}}{\abs{T}} \frac{2^k}{24\lg(1/\alpha)} \frac{\abs{T}}{2^k} = \frac{\abs{S}}{24\lg(2/\alpha)} \;.
  \end{align*}
  Now suppose on the other hand that we have a \nnew{nice iteration}
  \begin{equation*}
    \begin{tikzcd}[column sep=huge]
      w^{(j)} \ar[r,"\Delta^{(j+1)}w"] & w^{(j+1)} \;,
    \end{tikzcd}
  \end{equation*}
  where $w^{(j)}\in A_k$ and $w^{(j+1)} \in A_{k+r}$ for some $r\geq 1$.
  Then, by Corollary \ref{cor:stepineq} and the definition of $A_k$, we get that
  \begin{equation*}
    \Delta^{(j+1)}w(S) \leq \frac{w^{(j)}(S)}{w^{(j)}(T)} \cdot \frac{1}{24\lg(2/\alpha)}\Delta^{(j+1)}w(T) \leq \frac{\abs{S}}{\abs{T}}\frac{2^k}{24\lg(2/\alpha)}\Bigl( \frac{1}{2^{k-1}} - \frac{1}{2^{k+r}}  \Bigr)\abs{T} \leq 2 \frac{\abs{S}}{24\lg(2/\alpha)} \;.
  \end{equation*}
  Now suppose that we have gotten to  \nnew{iteration} $m$ via a sequence of nice \nnew{iterations}
  \begin{equation*}
    \begin{tikzcd}[column sep=huge]
      w^{(0)}\ar[r,"\Delta^{(1)}w"] & w^{(1)} \ar[r,"\Delta^{(2)}w"] & \cdots  \ar[r,"\Delta^{(m)}w"] & w^{(m)} \;,
    \end{tikzcd}
  \end{equation*}
  where $w^{(i)}(S) \geq 3\abs{S}/4$ for all $i\leq m-1$.
  We want to show that $w^{(m)}(S) \geq 3\abs{S}/4$.
  First, we note that $w^{(m-1)} \in A_k$ for some $k \leq \lg(2/\alpha)$, since
  $1/2^{\lg(2/\alpha)} = 1/(2/\alpha) = \alpha/2$
  and $w^{(m-1)}(T) \geq 3\abs{S}/4 > \abs{S}/2 \geq \alpha \abs{T}/2$.
  So we get that we have at most $k+1$  \nnew{iterations} $A_\ell \to A_{\ell+r}$ $(r \geq 1)$
  (the worst case being $A_0 \to A_1 \to \cdots \to A_k \to A_{k+r}$),
  and thus by the above
  \begin{equation*}
    \littlesum_{s=0}^{m-1} \Delta^{(s+1)}w(S) \leq \frac{3(k+1)\abs{S}}{24\lg(2/\alpha)} \leq \frac{(\lg(2/\alpha)+1)\abs{S}}{8\lg(2/\alpha)} \leq \frac{2\lg(2/\alpha)\abs{S}}{8\lg(2/\alpha)} \leq \frac{\abs{S}}{4} \;.
  \end{equation*}
  Therefore, we have that
  \begin{equation*}
    w^{(m)}(S) = \abs{S} - \littlesum_{s=0}^{m-1} \Delta^{(s+1)}w(S) \geq 3\abs{S}/4 \;.
  \end{equation*}
%  i.e., $w^{(m)}(S) \geq 3\abs{S}/4$.
  Hence, by Corollary \ref{cor:stepineq}, we get another nice iteration.

  By induction, every sequence of nice \nnew{iterations}
  \begin{equation*}
    \begin{tikzcd}[column sep=huge]
      w^{(0)}\ar[r,"\Delta^{(1)}w"] & w^{(1)} \ar[r,"\Delta^{(2)}w"] & \cdots  \ar[r,"\Delta^{(m)}w"] & w^{(m)}
    \end{tikzcd}
  \end{equation*}
  satisfies $w^{(i)}(S) \geq 3\abs{S}/4$ for all $i$, and there is a sequence of purely nice iterations.

  Hence, every weight function $w$ in a sequence of nice \nnew{iterations}
starting from $w^{(0)}$ satisfies $w(T) \geq w(S) \geq 3\abs{S}/4 > \abs{S}/2 \geq \alpha\abs{T}/2$ (since $T$ is $\alpha$-good), and thus we cannot get to an \nnew{iteration} where $w(T) < \alpha\abs{T}/2$.
This implies that we have to exit the algorithm by getting ``YES'' in the \textsc{BasicMultifilter}
(since every branch of the algorithm terminates). This completes the proof.
\end{proof}

\nnew{
The analysis of the runtime can be found in the following section, where it is also shown
that the output list is of size $O(1/\alpha^2)$.
In Appendix~\ref{app:list-reduction}, we show how to efficiently
post-process the output to an $O(1/\alpha)$-sized list.}

\subsection{Runtime Analysis} \label{ssec:runtime}

In this section, we provide a detailed runtime analysis of our main algorithm.
We start with the following simple lemma.

\begin{lemma}\label{lem:filtertime}
\textsc{BasicMultifilter} has worst-case runtime
$\tilde{O}(\abs{T}d)$.
\end{lemma}
\begin{proof}
The operations in \textsc{BasicMultifilter} can be implemented efficiently
with an appropriate preprocessing. In particular, computing $v\cdot x$ for each $x\in T$ can be done in $O(d|T|)$ time, and then sorting these values can be done in $O(|T|\log(|T|))$ time. Then in $O(|T|)$ time, using a linear scan, we can compute and store $w(\{x\in T: v\cdot x \leq v\cdot y\})$ for all $y\in T$.
Computing $a$ and $b$ in Step~\ref{step:bm-quantile} can be done in $O(\log(|T|))$ time by binary search.
The variance of $v\cdot T$ over $T$ conditioned on $I$ can be done in linear time in the usual way.
The computation of $w_{new}$ in Step~\ref{step:bm-if}(b) is easily done in linear time.

The most challenging part of the algorithm is Step~\ref{step:bm-two-weights}.
Assuming that there is a solution, we let $a$ be the smallest value in $v\cdot T_1$ and $b$ the largest value in $v\cdot T_2$.
We will have our algorithm guess which of $w(T_1)$ or $w(T_2)$ is larger, thus determining which term achieves the minimum
in Equation~\eqref{eq:minineq}. If $w(T_1)$ is larger, we additionally guess the value of $a$ and if $w(T_2)$ is larger, we guess $b$.
We note that there are $O(|T|)$ many possible outcomes for these guesses and that upon making them we can determine
the value of $\min(1-w(T_1)/w(T),1-w(T_2)/w(T))$. This lets us determine the largest possible value of $R$ consistent with condition \eqref{eq:minineq}. This in turn lets us determine the smallest possible value of $b$ (if we guessed $a$),
or the largest possible value of $a$ (if we guessed $b$) consistent with these guesses,
and condition \eqref{eq:minineq} by using binary search to find the largest/smallest element of $v\cdot T$
so that $b-a\geq 2R$ and so that the $w(T_i)$ chosen to attain the minimum actually does.
Note then that if any choices of $t$ and $R$ consistent with our guess and with condition \eqref{eq:minineq}
are also consistent with condition \eqref{eq:Tineq}, this extreme choice will be. Therefore, it suffices
for each of these $O(|T|)$ possible guesses to spend $O(\log(|T|))$ time to find this extreme value,
and then spend $O(1)$ time to verify whether or not condition \eqref{eq:Tineq} and \eqref{eq:minineq} hold.
Once we find some choice for which they do, we can return that one.
The total runtime for this step is at most $O(|T|\log(|T|))$.
\end{proof}

\begin{lemma}\label{lem:subroutinetime}
  \textsc{MainSubroutine} has worst case runtime $\tilde{O}(\abs{T}d)$.
\end{lemma}
\begin{proof}
We note that Steps~1 and~2 can be implemented in time $\tilde{O}(\abs{T}d)$ by standard methods.
In particular, we do not need to explicitly compute the weighted empirical covariance. We can instead
use power-iteration to find an approximately largest eigenvalue-eigenvector pair in $\tilde{O}(\abs{T}d)$ time.
Even though this computation is randomized, we can ignore the error probability for the following reason:
By standard linear-algebraic tools (see, e.g., Fact~5.1.1 of~\cite{li18thesis}), this computation takes time
$\tilde{O}(\abs{T}d \ \log(1/\delta))$, where $\delta$ is the error probability. Since we only use this subroutine
$|T|$ many times, we can take $\delta \ll 1/|T|$ and use a union bound. This incurs at most a logarithmic
overhead in the running time.

By Lemma~\ref{lem:filtertime}, Step~3 can be completed in $\tilde{O}(\abs{T}d)$ time, which completes the proof.
\end{proof}

We are now ready to prove the main theorem of this section:

\begin{theorem}\label{thm:runtime-final}
\textsc{List-Decode} has worst-case runtime $\tilde{O}(\abs{T}^2d/\alpha^2)$
and the output list $M$ has size at most $4/\alpha^2$.
\end{theorem}
\begin{proof}
We note that the multi-filter algorithm gives us the structure of a tree, wherein,
by Equation~\eqref{eq:Tineq}, we get that
  \begin{equation*}
    \abs{T}^2 = w^{(0)}(T)^2 \geq \littlesum_{\text{all leaves }w} w(T)^2 \geq \littlesum_{\text{all leaves }w} (\alpha\abs{T}/2)^2\;,
  \end{equation*}
  and thus
  \begin{equation*}
    4/\alpha^2 \geq \littlesum_{\text{all leaves }w} 1 = \#\text{of leaves}.
  \end{equation*}
  So we have at most $4/\alpha^{2}$ leaves and thus at most $4/\alpha^2$ elements in the list $M$. Also, by the above, we never have more than $O(\alpha^{-2})$ branches at a given depth in the tree.

The bottleneck of the algorithm is clearly in Step~2. Each call to \textsc{MainSubroutine} can be completed
with runtime $\tilde{O}(\abs{T}d)$ by Lemma~\ref{lem:subroutinetime},
so we just need to consider how many rounds of Step~2 we can have in the worst case.
We note that each iteration of \textsc{BasicMultifilter} sets at least one weight to $0$ (in every branch),
so the tree has depth at most $\abs{T}$, and therefore we run Step~2 at most $O(\abs{T}\alpha^{-2})$ times,
since we never have more than $O(\alpha^{-2})$ branches. Hence, the runtime is $\tilde{O}(\abs{T}^2d\alpha^{-2})$.
\end{proof}

\section{Conclusions} \label{sec:conc}

In this paper, we study the problem of list-decodable mean estimation for bounded covariance distributions.
As our main contribution, we give the first provable practical algorithm for this problem with near-optimal error guarantees.
At a technical level, our work strengthens and generalizes the multi-filtering approach of~\cite{DiakonikolasKS18-mixtures}, which had focused on spherical Gaussians, to apply under a bounded covariance assumption. 
This work is part of the broader agenda of developing fast and practical algorithms for list-decodable learning under
minimal assumptions on the inliers. 

The obvious open problem is to design faster provable algorithms 
for list-decodable mean estimation with $\tilde{O}(n d)$ as the ultimate goal. 
The runtime analysis of our algorithm gives a bound of $\tilde{O}(n^2 d  / \alpha^2)$.
We believe this can be easily improved to $\tilde{O}(n^2 d /\alpha^{1+c})$, for any constant $c>0$. 
A bottleneck in our runtime analysis comes from the number of recursive subsets that our algorithm needs to run on. 
This is controlled by Equation \eqref{eq:Tineq}, which postulates that $\sum w_i(T)^2 \leq |T|^2$. This 
condition ensures that we have no more than $O(\alpha^{-2})$ many subsets at any given time. 
This can be improved by replacing \eqref{eq:Tineq} by $w(T_1)^{1+c}+w(T_2)^{1+c} \leq w(T)^{1+c}$,
for any $c>0$. We believe this should suffice to let the remainder of our analysis go through 
and reduce the $\alpha$-dependence of our runtime to $O(\alpha^{-1-c})$.

The concurrent work~\cite{CMY20} gives an SDP-based algorithm whose runtime is $\tilde{O}(n d) / \poly(\alpha)$,
i.e., near-optimal as a function of the dimension $d$, but suboptimal (by a polynomial factor) as a function of $1/\alpha$. 
We note that the dependence on $1/\alpha$ is equally significant in some of the key applications 
of list-decodable learning (e.g., in learning mixture models). 
Can we obtain a truly near-linear time practical algorithm?

\paragraph{Acknowledgements.}
We thank Alistair Stewart for his contributions in the early stages of this work.

\bibliographystyle{alpha}
\bibliography{allrefs}

\appendix

\section{Efficient List Size Reduction} \label{app:list-reduction}

Here we give a simple and efficient method to reduce the list of hypotheses
to one of size $O(1/\alpha)$.

\begin{theorem}\label{thm:list-reduction}
There exists an algorithm that given the output of Theorem~\ref{thm:main-det} runs
in time $O(d/\alpha^3)$ and returns a list of $O(1/\alpha)$ hypotheses
with the guarantee that at least one of the hypotheses are within
$O(\log(1/\alpha)/\sqrt{\alpha})$ of $\mu$, assuming that one of the
hypotheses of the original algorithm was.
\end{theorem}

The algorithm here is quite simple. We set $C>0$ to be a
sufficiently large universal constant and find a maximal subset of our
hypotheses that are pairwise separated by at least
$C\log(1/\alpha)/\sqrt{\alpha}$. This can be done by starting with an
empty set $H$ of hypotheses and for each hypothesis in the output of
Theorem~\ref{thm:main-det} comparing it to each of the hypotheses currently in $H$
and adding it if it is not too close to any of them. It is clear that
the runtime of such an algorithm is at most $O(d |H|/\alpha^2)$. It is
also clear that if our original set of hypotheses contained a $\mu_0$,
then $H$ will contain a $\tilde \mu$ with $\|\tilde \mu - \mu_0\|_2 <
C\log(1/\alpha)/\sqrt{\alpha}$. Therefore, if our original set
contained a $\mu_0$ with $\|\mu_0-\mu\|_2 =
O(\log(1/\alpha)/\sqrt{\alpha})$, then by the triangle inequality, $H$
will contain a $\tilde \mu$ with $\|\tilde \mu-\mu\|_2 =
O(\log(1/\alpha)/\sqrt{\alpha})$.

All we have left to prove is that $|H| = O(1/\alpha)$. For this we
note that for each hypothesis $\mu_i$ that Theorem~\ref{thm:main-det} returns,
there is an associated weight function $w_i$ on $T$ so that
\begin{itemize}
\item $w_i(T) \geq \alpha |T|/2$,
\item $\Cov_{w_i}[T] \prec O(\log^2(1/\alpha)) I.$
\end{itemize}
It turns out that this is enough to show that $|H|=O(1/\alpha)$. This
argument has become fairly standard in the robust list-decoding
literature, but unfortunately, we cannot find an existing theorem
statement that applies to exactly our case. The techniques in
the proof of Claim 5.2 of~\cite{DiakonikolasKS18-mixtures} are very similar.
We state a general theorem here that not only covers our case, but
should cover more general settings:
\begin{lemma}
Let $T$ be a subset of $\R^d$ and $\alpha, \sigma>0$ be real numbers.
Let $H$ be another subset of $\R^d$ such that for each $u \in H$ there
is a weight function $w_u$ on $T$ with $w_u(T) \geq \alpha \abs{T}$ and such
that for any unit vector $v \in \R^d$, $\wPr{w}_{x \in T}[|v\cdot(x-u)| > \sigma] < \alpha/10$.
Assume furthermore that for any $u, u'\in H$, we have $\norm{u-u'}_2 > 2\sigma$.
Then we have that $|H| \leq 2/\alpha$.
\end{lemma}

Applying this lemma to our set $H$ with the weight functions $w_i$ mentioned above
and $\sigma$ a sufficiently large multiple of
$\log(1/\alpha)/\sqrt{\alpha}$ yields Theorem~\ref{thm:list-reduction}.

Before we begin, we will introduce the notation $\bigcup_{i\in I}w_i$ and $\bigcap_{i\in I}w_i$ for the weight functions given by
\begin{align*}
  \Bigl( \bigcup_{i \in I} w_i \Bigr)(x) &= \max_{i \in I} w_i(x), \\
  \Bigl( \bigcap_{i \in I} w_i \Bigr)(x) &= \min_{i \in I} w_i(x)
\end{align*}
for any finite index set $I$ and any $x \in T$.

\begin{proof}
We proceed by contradiction. Assume that the above hypotheses hold and
that $|H| > 2/\alpha$. We note that $(\bigcup_{u \in H} w_u)(T) \leq |T|$. However, the
sum of the individual terms is much larger than this since $\sum_{u \in H} w_u(T) \geq \sum_{u \in H} \alpha \abs{T} \geq 2\abs{T}$ because $\abs{H} \geq 2/\alpha$. By restricting $H$ to a subset if necessary, we can guarantee that $\abs{H} = \lceil 2/\alpha \rceil$, which will still ensure that $\sum_{u \in H} w_u(T) \geq 2\abs{T}$. Next, as we will
show, the pairwise intersections of the $w_u$ are small. This will
give a contradiction.

To start with, note that given $u, u' \in H$ we would like to show that
$(w_u \cap w_{u'})(T)$ is small. For this, we let $v$ be the unit vector
in the direction of $u-u'$. By assumption, $v\cdot (u-u') = \|u-u'\|_2 >
2\sigma$. Therefore, by the triangle inequality, for every $x\in T$ it will either be the case that $|v\cdot (x-u)| > \sigma$
or $|v\cdot (x-u')| > \sigma$. However, if we call these sets $D_u$ and $D_{u'}$, we have that $w_u(D_u) \leq \alpha/10 w_u(T)$ and $w_{u'}(D_{u'})\leq w_{u'}(T)$. Therefore, we have
that $(w_u \cap w_{u'})(T) \leq \alpha (w_u(T) + w_{u'}(T))/10.$

Given this, we wish to make use of approximate inclusion-exclusion. In particular, given some ordering over the points in $H$, we note that for any $x\in T$ we have that
$$
1 \geq \max_{u \in H} w_u(x) \geq \sum_{u\in H}w_u(x) - \sum_{u,u' \in H, u<u'} \min(w_u(x),w_{u'}(x)).
$$
Summing over $x\in T$, we find that
\begin{align*}
|T| & \geq \Bigl(\bigcup_{u\in H} w_u\Bigr)(T) \\
& \geq \littlesum_{u\in H} w_u(T) - \littlesum_{u,u'\in H, u < u'} (w_u \cap w_{u'})(T) \\
& \geq \littlesum_{u\in H} w_u(T) - \littlesum_{u,u'\in H, u < u'}
(\alpha/10)\bigl( w_u(T) + w_{u'}(T) \bigr) \\
& = \Bigl( \littlesum_{u\in H} w_u(T)\Bigr)(1 - (\alpha/10)(|H|-1))\\
& \geq \Bigl( \littlesum_{u\in H} w_u(T) \Bigr)(1-(\alpha/10)(2/\alpha))\\
& \geq 2|T|(8/10)\\
& > |T|,
\end{align*}
where we use that $\abs{H}-1 \leq 2/\alpha$ and $\sum_{u \in H} w_u(T) \geq 2\abs{T}$ (after we restricted $H$ above), yielding a contradiction. Hence $\abs{H} \leq 2/\alpha$.

This completes our proof.
\end{proof}

\end{document}